%% file: main.tex
\documentclass[letterpaper, 10 pt, conference]{ieeeconf}  
\IEEEoverridecommandlockouts                              
\usepackage{times}
\usepackage{geometry}
\usepackage{multicol}

\usepackage[bookmarks=true]{hyperref}
\usepackage{url}
\usepackage{graphicx}
\usepackage{amssymb}
\usepackage{todonotes}
\usepackage{mathtools}
\usepackage{algorithmicx}
\usepackage[noend]{algpseudocode}
\usepackage{algorithm}
\usepackage{booktabs}

\usepackage{amsthm}
\usepackage[subtle, bibbreaks=tight, paragraphs=tight, floats=tight, mathspacing=tight, wordspacing=tight, tracking=normal]{savetrees}
\usepackage[margin=0pt,font=footnotesize,labelsep=period]{caption}

\newtheorem{definition}{Definition}
\newtheorem{lemma}{Lemma}
\newtheorem{theorem}{Theorem}

\newcommand{\mrm}[1]{~\mathrm{#1}~}
\newcommand{\enumSet}[2]{\left\lbrace{#1}_1, {#1}_2 \ldots, {#1}_{#2} \right\rbrace}
\newcommand{\reals}{\mathbb{R}}
\newcommand{\integers}{\mathbb{Z}}

\newcommand*\Let[2]{\State #1 $\gets$ #2}
\algrenewcommand\algorithmicindent{1em}

\newcommand{\environment}{\mathcal{E}}

\newcommand{\envgraph}{G}
\newcommand{\traverse}{T}

\newcommand{\dzcliff}{\Delta Z_{cliff}}
\newcommand{\ksteep}{K_{steep}}
\newcommand{\cliffsFun}{\texttt{cliffs}}
\newcommand{\steepFun}{\texttt{steep}}

\newcommand{\blockSurface}{s}
\newcommand{\fwedge}{wedge_f}
\newcommand{\bwedge}{wedge_b}
\newcommand{\lblock}{L_{B}}
\newcommand{\structure}{T}

\newcommand{\strVector}{\mathbf{\hat{u}}}
\newcommand{\strCells}{C}
\newcommand{\cell}{c}
\newcommand{\cellCorners}{\mathbf{\lambda}}
\newcommand{\height}{h}
\newcommand{\terminator}{t}

\newcommand{\cost}{\texttt{cost}}
\newcommand{\regions}{R}
\newcommand{\region}{r}
\newcommand{\structures}{\mathcal{T}}
\newcommand{\regionGraph}{G_R}
\newcommand{\buildpoints}{B}
\newcommand{\buildpoint}{\mathbf{b}}
\newcommand{\waterfall}{\operatorname{\textsc{Waterfall}}}
\newcommand{\cellmin}{h_{min}}
\newcommand{\cellmax}{h_{max}}

\newcommand{\mst}{M}
\newcommand{\bbmst}{\operatorname{\textsc{BB-MST}}}
\newcommand{\solution}[2]{\mst^*_{#1, #2}}
\newcommand{\powerset}{\mathcal{P}}
\newcommand{\exTableChair}{(A)}
\newcommand{\exPushedIn}{(B)}
\newcommand{\exCheckerboard}{(C)}
\newcommand{\exOffice}{(D)}
\newcommand{\exStairs}{(E)}
\begin{document}

\title{\vspace{-0.8cm}Optimal Structure Synthesis for Environment Augmenting Robots}

\author{
Tarik~Tosun, Cynthia Sung, Colin McCloskey, and Mark Yim%
  \thanks{Tosun, Sung, and Yim: University of Pennsylvania, Philadelphia PA}%
  \thanks{McCloskey: Yale University, New Haven CT}%
  \thanks{\texttt{\{tarikt, crsung, yim\}@seas.upenn.edu}}%
  \thanks{\texttt{colin.mccloskey@yale.edu}}%
  \thanks{Work funded by NSF grant numbers CNS-1329620 and CNS-1329692}%
}
\maketitle%
\begin{abstract}
Building structures can allow a robot to surmount large obstacles, expanding the set of areas it can reach.
This paper presents a planning algorithm to automatically determine what structures a construction-capable robot must build in order to traverse its entire environment. 
Given an environment, a set of building blocks, and a robot capable of building structures, we seek a optimal set of structures (using a minimum number of building blocks) that could be built to make the entire environment traversable with respect to the robot's movement capabilities.  
We show that this problem is NP-Hard, and present a complete, optimal algorithm that solves it using a branch-and-bound strategy.  The algorithm runs in exponential time in the worst case, but solves typical problems with practical speed.  In hardware experiments, we show that the algorithm solves  3D maps of real indoor environments in about one minute, and that the structures selected by the algorithm allow a robot to traverse the entire environment. An accompanying video is available online at \url{https://youtu.be/B9WM557NP44}.
\end{abstract}

\section{Introduction}
\label{sec:introduction}
%
Augmenting the environment to make a task easier is a familiar human experience: to reach objects on a shelf, we use a stepstool, and at a larger scale, we construct bridges to cross wide rivers.
In contrast most robotic planning strategies  search for behaviors the robot can execute  within the constraints of its environment, without considering how those constraints could be changed.
Well-established work in construction robotics provides evidence that robots have the ability to build large structures \cite{petersen2011termes, Werfel2007, Seo2013}, indicating that they should, similarly to humans, be able to augment their environment for simpler task execution.

 Consider a scenario where a small robot must move through an environment filled with objects larger than itself. 
The robot is unable to cross over large objects or gaps. As a result, entire portions of the environment may be  inaccessible. 
The planning problem can be posed: Given an environment, a robot, and a supply of building blocks, can we find a set of structures that could be added to the environment to make it fully accessible to the robot (i.e. there exists a navigable path between any pair of points)?  Furthermore, since structure-building is a time-consuming process, can we find such a set of structures which uses a minimum number of building blocks?  We refer to this as the \emph{optimal structure synthesis} problem.

This paper presents a mathematical formalism for optimal structure synthesis,  shows that the problem is NP-Hard, and   presents a complete, optimal algorithm that will find a minimum-cost solution to any instance of the problem, if a solution exists.  The algorithm solves practical problems efficiently using a branch-and-bound strategy, typically exploring a tiny fraction of the exponentially-large solution space before finding the optimal solution.  In our experiments, we show that the algorithm finds optimal solutions in about one minute for 3D maps of real indoor environments, and demonstrate that the structures selected by the algorithm do indeed allow the robot to access every region of the environments.
%
%
\section{Related Work}
\label{sec:related_work}
The fields of collective construction robotics and modular robotics offer examples of systems that build and traverse structures.  Petersen et al. present Termes \cite{petersen2011termes}, a collective construction robot system that creates structures using blocks co-designed with the robot. Werfel et al. present algorithms for environmentally-adaptive construction: a team of robots senses obstacles and builds around them, modifying the goal structure if needed to leave room for immovable obstacles \cite{Werfel2007}. Other work allows robots to adaptively reinforce the structure they are building to bear the load imposed on it \cite{melenbrink2018local}.  Terada and Murata \cite{terada2004automatic} present a lattice-style modular system with two parts, structure modules and an assembler robot. Like many lattice-style modular systems, the assembler robot can only move on the structure modules, and not in an unstructured environment. M-blocks \cite{romanishin20153d} form structures out of robot cubes which rotate over the structure, and can reconfigure between arbitrary 3D shapes, except those containing certain inadmissible sub-configurations \cite{sung2015reconfiguration}. 
Other related work in manipulation planning allows robots to carry out multi-step procedures to assemble furniture \cite{knepper2013ikeabot} or rearrange clutter surrounding a primary manipulation task \cite{dogar2012planning}.
 
There is also some work showing that robots can deploy structures to enhance their ability to move through an environment.  Napp et al. present a distributed algorithm for adaptive ramp building with amorphous materials \cite{napp2014distributed}.  Using local information, the algorithm controls one or more  robots to deposit amorphous material (like foam or sandbags) on their environment to make a goal point accessible \cite{saboia2018autonomous}.  Our work addresses a similar, complementary problem: assuming a map of the entire environment, we generate a {\em globally optimal} set of structures making {\em every} point accessible.

We build on prior work with the SMORES-EP robot (Figure \ref{fig:env-aug}), in which the robot autonomously built and deployed structures to overcome obstacles while completing tasks \cite{TT:JD:GJ:HK:MC:MY:icra:18}. This prior system relied on a human-made library  of structures: for example, the library entry ``\texttt{height-2 ramp}'' would have an associated classifier to identify environments of type ``\texttt{height-2 ledge},'' indicating that such a ledge can be made traversable by adding the ramp.  In contrast, our current paper provides a general formalism to decide what structure the robot should build, and where it should build it.
\begin{figure}
\centering
\includegraphics[width=0.32\columnwidth]{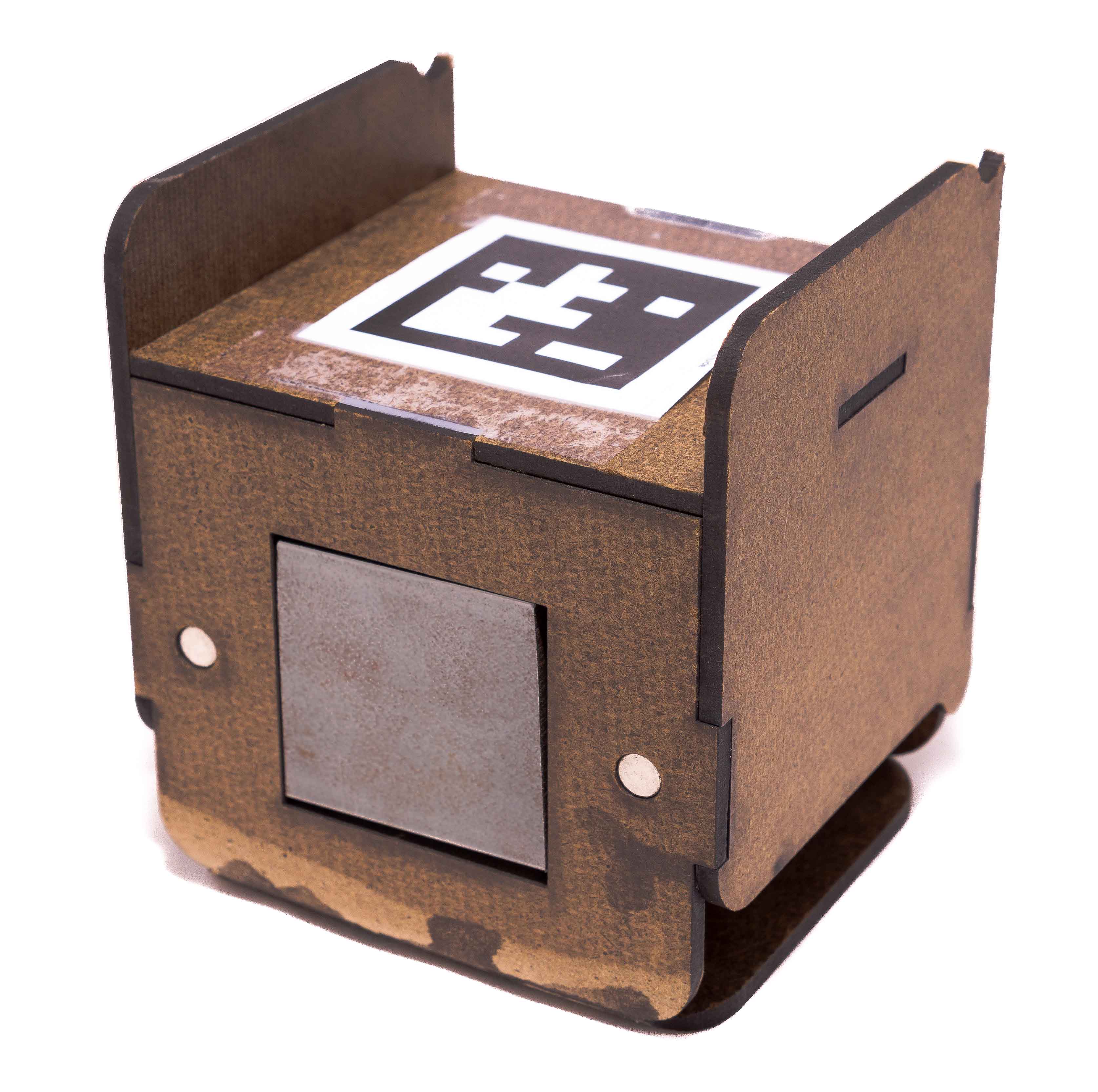}
\includegraphics[width=0.32\columnwidth]{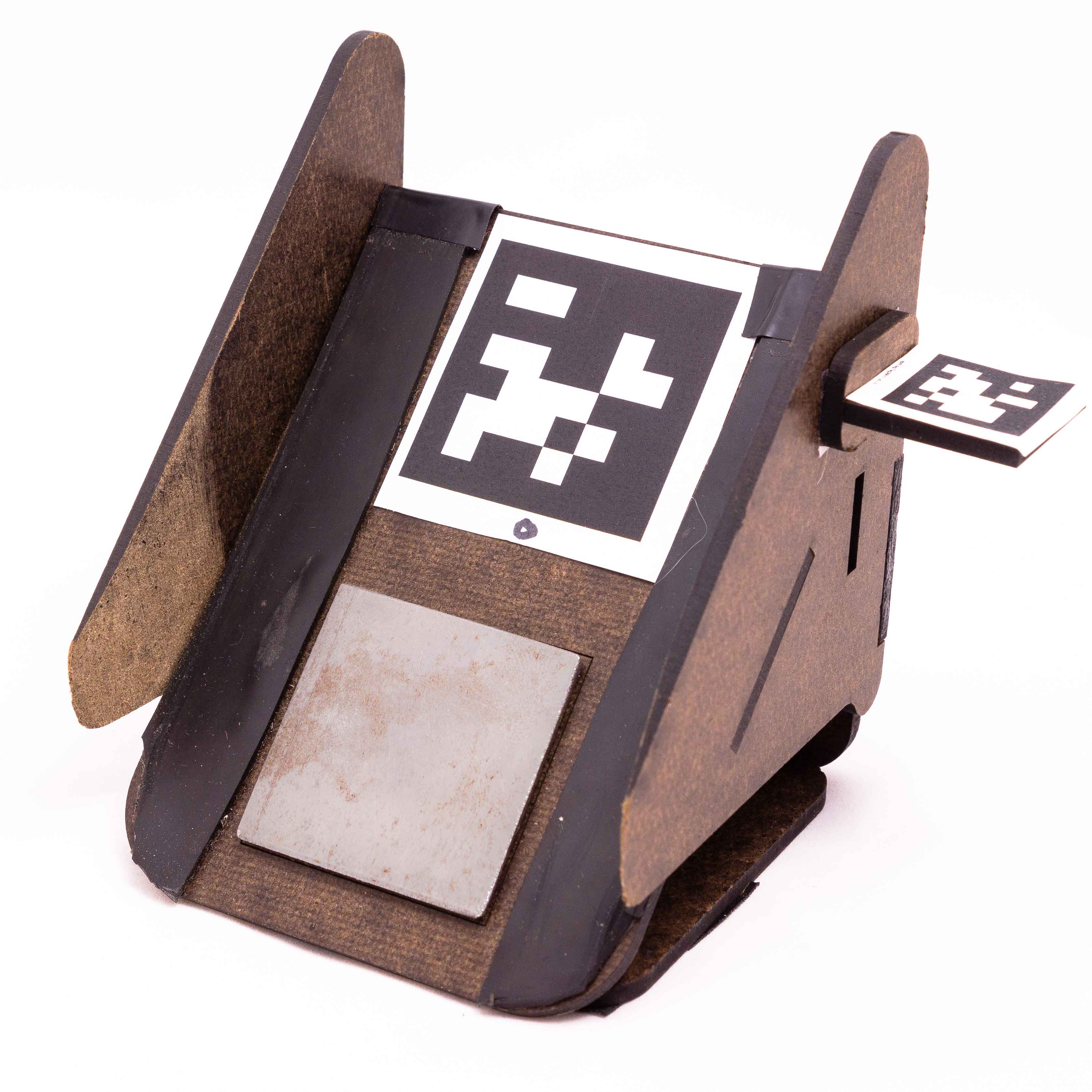}
\includegraphics[width=0.32\columnwidth]{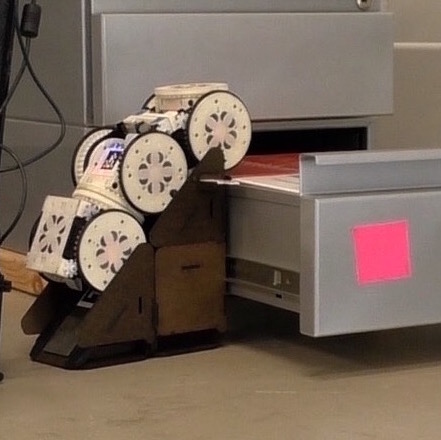}
\caption{The SMORES-EP robot can use the block and wedge building blocks shown here to build structures that make it possible to surmount obstacles \cite{TT:JD:GJ:HK:MC:MY:icra:18}.}
\label{fig:env-aug}
\vspace{-0.5cm}
\end{figure}
\section{Problem Formulation}
\label{sec:formulation}
\subsection{Problem Summary}
Given an \textit{environment}, a \textit{robot}, and a set of \textit{building blocks}, we seek a set of \textit{structures}, made of building blocks, that could be added to the environment to make it fully accessible by the robot.  We present an algorithm that is guaranteed to find a globally optimal solution to this discrete, combinatorial optimization problem. The solution is optimal in the sense that it uses a minimum number of building blocks.

The \textit{environment} is a discretized height map -  a 2d x-y grid with a single height Z at each cell (Fig.~\ref{fig:problem-overview}). The \textit{robot} moves over the surface of the environment.  We abstract the motion of the robot as discrete transitions along the edges of the graph representing the environment. The movement capabilities of the robot are modeled as the \textit{traversability criteria}, which quantify the largest obstacles and steepest slopes the robot can move over, and consequently determine whether the robot can occupy a node or move along an edge in the environment graph. These criteria are detailed in Definitions~\ref{def:non-traversable-edge} and~\ref{def:non-traversable-node}. The framework is general to any robot whose motion can be modeled using these criteria.

\textit{Building blocks} can be assembled by the robot to form \textit{structures}. Our goal is to determine the shapes and locations of a set of structures that could be added to the environment to make it fully traversable. Section~\ref{sec:structures} provides a formal definition for structures made of building blocks, and describes the constraints that valid structures must obey.

Our algorithm proceeds in two stages.  First, the Waterfall algorithm (Section~\ref{sec:waterfall}) finds the set of min-cost structures connecting each pair of regions in the environment. Among these, the BB-MST algorithm (Section~\ref{sec:bbmst}) then finds a min-cost tree of structures that spans all regions, while obeying the constraints for structure validity.
Together, Waterfall and BB-MST generate the minimum-cost set of structures that makes the environment fully traversable (Theorem~\ref{thm:main-result}).
\begin{figure}
\begin{center}
 \includegraphics[width=0.32\columnwidth]{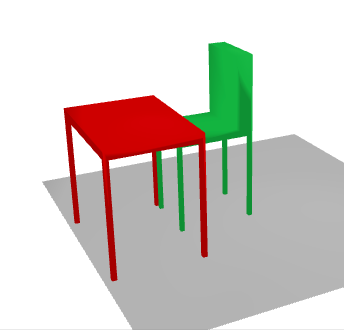}
 \includegraphics[width=0.32\columnwidth]{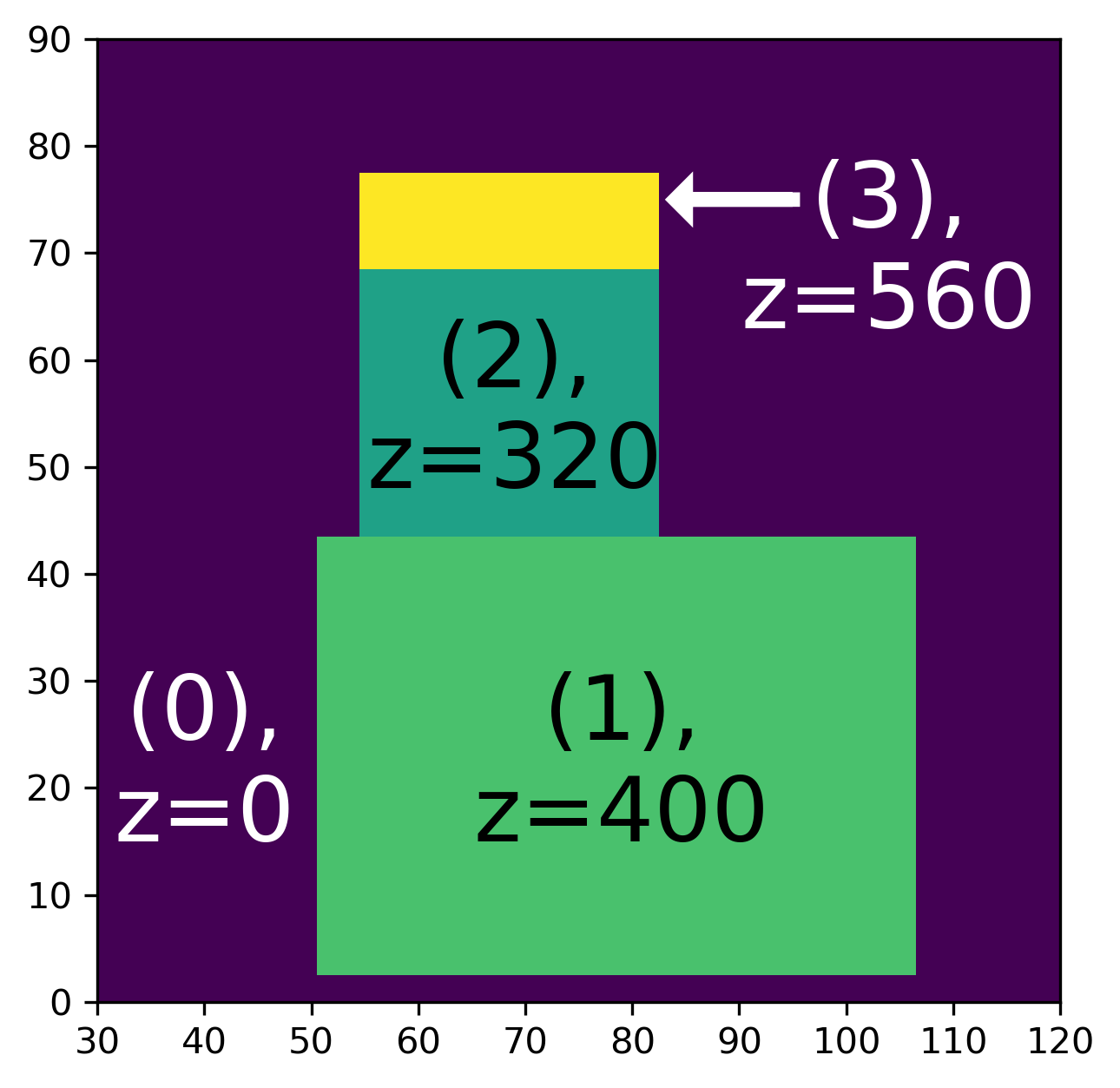}
 \includegraphics[width=0.32\columnwidth]{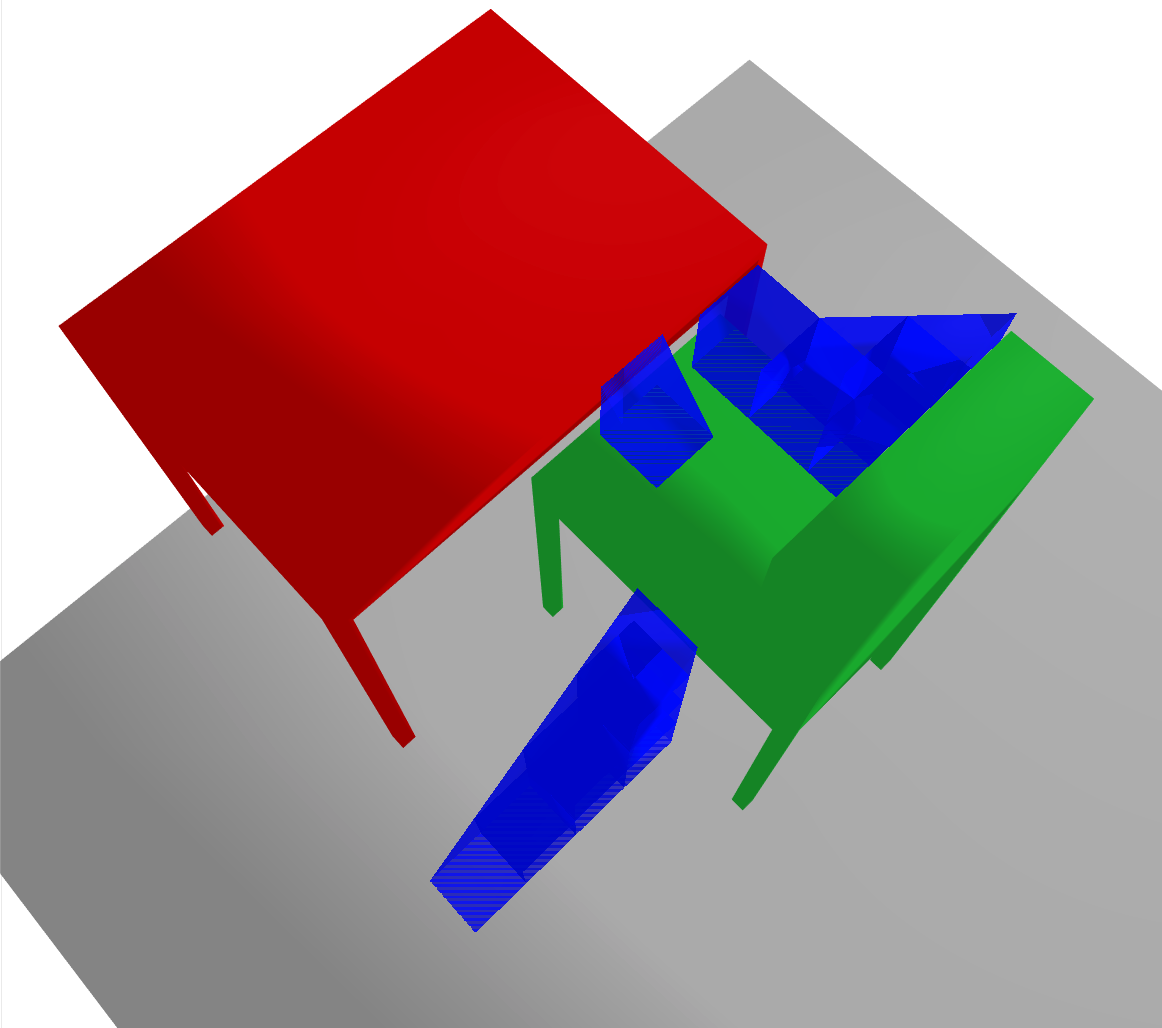}
\end{center}
 \caption{\textbf{Left:} Example environment (table and chair).  \textbf{Center:} Height field representation (top down view), with labelled regions and heights. \textbf{Right:} Optimal solution generated by our algorithm, which makes all four regions accessible.  Structures are blue.}
\label{fig:problem-overview}
\vspace{-0.5cm}
\end{figure}

\subsection{Preliminaries}


Consider a ground robot traversing 
 a discretized height-map environment $\environment:\integers^2\to\reals^+$. 
The environment can be represented as a grid graph $\envgraph(V,E)$, where $V$ is a set of nodes corresponding to each of the grid cells, and edges $E\subseteq V\times V$ connect neighboring cells.
The robot's movements are treated as discrete transitions along edges.
The robot is subject to certain physical limitations,  which determine its ability to  traverse this graph, defined for both nodes and edges.

The first criterion is on edges, or transitions between nodes. A ground robot is typically not able to traverse a sharp rise or dropoff (or ``\emph{cliff}''). Edges across which the difference in height is above a threshold $\Delta Z_{cliff}$ is thus non-traversable.

\begin{definition}[Non-Traversable Edge] \label{def:non-traversable-edge}
An edge $e=(v_i, v_j)$ is non-traversable if $|\environment(v_i) - \environment(v_j) | > \dzcliff$. We refer to these edges as cliffs.
\end{definition}

The second criterion is on the slope of the environment. 
We find slopes too ``\emph{steep}''  by identifying nodes at which the average gradient (within a window of radius $d$ about the node) is higher than a threshold $\ksteep$.  To prevent cliffs from influencing the apparent steepness of nearby nodes (due to the windowed average), we remove the contribution of cliffs from the gradient before identifying steep nodes. Let:
\begin{eqnarray}
\cliffsFun(\environment(v)) =&
  \begin{cases}
    |\nabla \environment(\vec{v})| &\mrm{if} |\nabla \environment(v)| > \dzcliff \\
    0 &\mrm{otherwise} 
  \end{cases} \\
\steepFun(\environment(v)) =& \underset{d}{ma}\left(|\nabla \environment(v)| - \cliffsFun(\environment(v)\right)
\end{eqnarray}
where $\underset{d}{ma}$ denotes a windowed-moving-average with window-size $d$. In our implementation, $d=8$~cm.

\begin{definition}[Non-Traversable Node] \label{def:non-traversable-node}
A node $v \in V$ is non-traversable if $\steepFun(\environment(v)) > \ksteep$.
\end{definition}
For the robot made of SMORES-EP modules used in our implementation, $\dzcliff=4$~cm and $\ksteep=1$.  These values were experimentally determined by testing the robot's ability to drive over sloped surfaces and step obstacles.

We define the \emph{traversable environment} to be $\envgraph_\traverse$, the subgraph of $\envgraph$ containing only traversable nodes and edges.
If $\envgraph_\traverse$ is connected, then the robot is able to reach every traversable grid cell.
However, it is possible that $\envgraph_\traverse$ is not connected. In this case, the environment is split into \emph{regions}, which are the connected components of $\envgraph_\traverse$.

\begin{definition}[Region]
  A region $\region_i \in \regions$ is a connected component in $\envgraph_\traverse(V,E)$.  By definition, a path exists between any two nodes in a single region, and no path exists between any two nodes in different regions.
\end{definition}

Fig.~\ref{fig:problem-overview} shows an environment with a table and chair (left), and  the corresponding height map $\environment$ (center). Different colors on this map correspond to different regions.

\subsection{Structures}
\label{sec:structures}
Structures are composed of building blocks, which the robot is able to carry, place, and drive over, and can be added to the environment can  allow the robot to move between
regions. Our goal is to select the shapes, positions, and orientations of structures that could be added to the environment to make it globally traversable.  We do not attempt to determine what actions the robot could take to actually build the structure - to do so, we could use an existing algorithm for robotic assembly planning (See Section \ref{sec:related_work}). 

A \textit{structure} is a line of columns of stacked \textit{block} and \textit{wedge} building blocks, shown in Figure \ref{fig:env-aug}.  Building blocks have a square footprint with side length $\lblock$; for our building blocks, $\lblock=8cm$. Importantly, we assume the robot can only move over the structure along a straight line - this is a physical constraint imposed by the building blocks, which have side walls.
We identify building blocks with type labels $t \in \lbrace block, \fwedge, \bwedge \rbrace$, with ``f'' and ``b'' denoting the two possible orientations (forward and backward) of wedges with respect to the structure. For each type, we define a surface function $\blockSurface: [0, \lblock] \to \reals^+$ describing its shape: \mbox{$block:~s(x)=\lblock$}, \mbox{$\fwedge:~s(x)=x$}, and \mbox{$\bwedge:~s(x)=\lblock-x$}.

Viewed from above, the footprint of a structure in the $xy$ plane is a linear array of $n$ contiguous squares called \textit{structure cells} (Fig.~\ref{fig:example-environment}, left); sitting atop each structure cell is a column of stacked blocks which comprise the structure.  Structures are not locked to the grid - rather, the build point $\buildpoint$ defines the position of the first structure cell in the plane, and the orientation vector $\strVector$ defines the line along which all cells lie, and the direction along which the robot may move.
\begin{definition}[Structure]
  A structure $\structure=\left\langle \buildpoint, \strVector, \strCells \right\rangle$
  has build point \mbox{$\buildpoint \in \mathbb{R}^2$},
  orientation vector $\strVector \in \mathbb{R}^2: |\strVector|=1$,
  and cells $\strCells = \lbrace \cell_1 \ldots \cell_n \rbrace$.
  Each cell \mbox{$\cell_i = \left\langle \cellCorners_i, \height_i, \terminator_i  \right \rangle, ~i \in \lbrace 1 \ldots n \rbrace$} has
  corners $\cellCorners_i \in \mathbb{R}^{4\times2}$, 
  column height $h_i \in \mathbb{Z}^+$,
  and surface block type $t_i$.
\end{definition}
The \textit{height} and \textit{surface block type} of a cell respectively specify the number of blocks stacked on that cell and the type of building block at the top of the stack. 
To fully define a structure, it is sufficient to provide $\buildpoint$, $\strVector$, and lists of cell heights $\enumSet{h}{n}$ and surface block types $\enumSet{t}{n}$; assuming $\lblock$ is known, cell corners $\cellCorners$ can be computed from $\buildpoint$ and $\strVector$. The \textit{cost} of a structure is the number of blocks
it contains, $\cost(\structure) = \sum_{i = 1 \ldots n} \height_i$.

For a structure to be considered valid, the environment surface between each of its cells must be suitable to support a column of blocks.  
For our structures, we define \textit{buildability} in terms of the flatness of the underlying environment surface.  Letting \mbox{$\environment(\cell_i) = \texttt{median}(\environment(v)) ~\forall~ v \in \cell_i$}, cell $\cell_i$ is buildable iff
\mbox{$\environment(v) - \environment(\cell_i) < \alpha \lblock~\forall~v \in \cell_i$}.  In our implementation, $\alpha=0.4$.

For a structure to be useful, it must be both buildable and traversable.  Since each block is individually traversable, we determine traversability of a structure by evaluating the cliff condition between neighboring cells $\cell_i$ and $\cell_{i+1}$:
\begin{eqnarray}
  \label{eqn:Zi}
  Z_i =& \environment(\cell_i) + \height_i * \lblock + \blockSurface_i(\lblock)\\
  \label{eqn:Zip1}
  Z_{i+1} =& \environment(\cell_{i+1}) + \height_{i+1} * \lblock + \blockSurface_{i+1}(0) \\
  \label{eqn:traversable}
  |Z_{i+1}-Z_i| <& \dzcliff
\end{eqnarray}
The boundary is traversable if Equation~\ref{eqn:traversable} is satisfied.

In addition to moving between blocks on the structure, the robot must also be able to transition from the structure to the ground surface at both ends, so we require the first and last cell of traversable structures to have zero height
and exempt them from the buildability condition. 
Additionally, when moving to the first or last structure cell involves crossing a region boundary (e.g. moving from the structure to the top of a cliff), the region boundary must be flat, and its surface normal must align with the structure orientation $\strVector, $ so that structure presses flat up against the cliff (Fig~\ref{fig:example-environment}). 
\begin{definition}[Valid Structure]
\label{def:valid-structure}
  A structure $\structure$ is valid if: $~\forall~ \cell_i, \cell_{i+1} \in \strCells$, $(\cell_i, \cell_{i+1})$ is traversable, and $\forall \cell_i \in \{\cell_2 \ldots \cell_{n-1}\}$, $\cell_i$ is buildable.
\end{definition}
\subsection{Conflicts}
\label{sec:conflicts}
%
Introducing a valid structure whose endpoint cells are in different regions creates a path between them.
Our goal, then, is to synthesize a set of structures $\structures = \enumSet{\structure}{k}$ on a world graph with regions $\regions$ such that there is a path between every pair of regions.
Because structures occupy physical space, certain combinations of them can  result in\ two kinds of \emph{conflicts}: (1) collision between two structures and (2) splitting regions with structures. 
%
\begin{definition}[Structure-Structure Conflict]
A pair of structures conflict if any of their cells intersect.
\end{definition}
\begin{definition}[Region-Structure Conflict]
   Consider region $\region_c$ and set of structures $\structures_c =\enumSet{\structure}{k}$. Let $V_{\region_c}$ be the set of nodes in region $\region_c$.
Let $V_{\structures_c}$ 
be the set of nodes in the cells of $\structures_c$.  Let $G_{\region_c}$ be the subgraph of $G(V,E)$ induced by $V_{\region_c} \setminus V_{\structures_c}$. There is a region-structure conflict between $\region_c$ and $\structures_c$ if $G_{\region_c}$ is disconnected.
\end{definition}
Structure-structure conflicts create physically impossible conditions and are thus never allowed. Region-structure conflict occur when one or more structures  split a region into multiple disconnected pieces, requiring each piece to be treated as a new region. 
%
\subsection{Problem Statement}
Consider an input environment $\environment$, represented as grid-graph $G(V,E)$, with regions $R$. 
Our objective is to find a min-cost, conflict-free set of structures which make the entire environment traversable.
Let $\regionGraph(\regions, \structures)$ be a graph in which nodes represent regions and edges represent valid structures connecting regions.  We seek a set of structures $\structures^*= \enumSet{\structure}{k}$ such that $\regionGraph(R,\structures^*)$ is fully connected and conflict-free, and $\cost(\structures^*)=\sum_{\structure_i \in \structures^*} \cost(\structure_i)$ is minimized.
%

%
%
\begin{figure}
 \includegraphics[width=0.4\columnwidth]{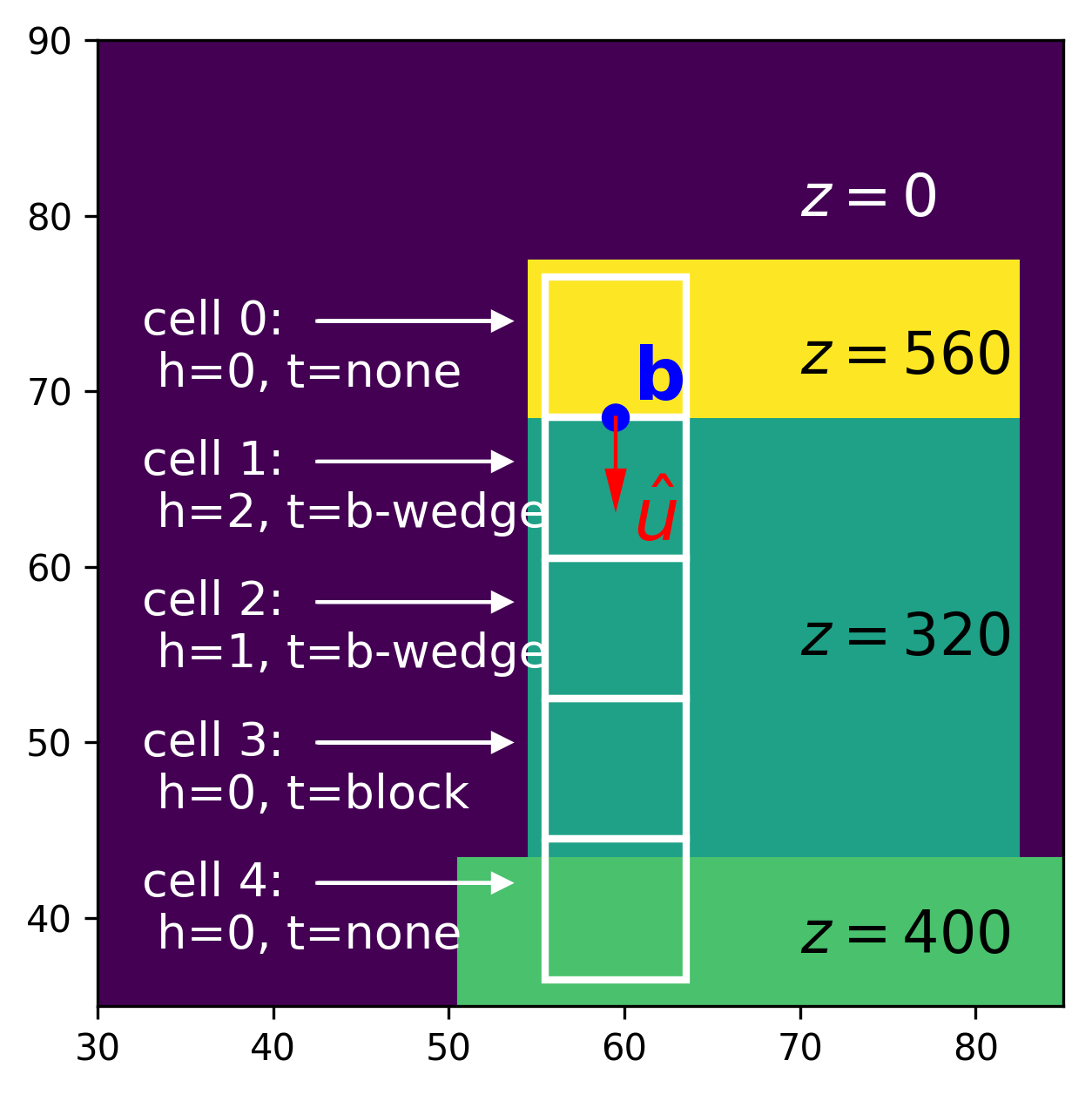}
 \includegraphics[width=0.26\columnwidth]{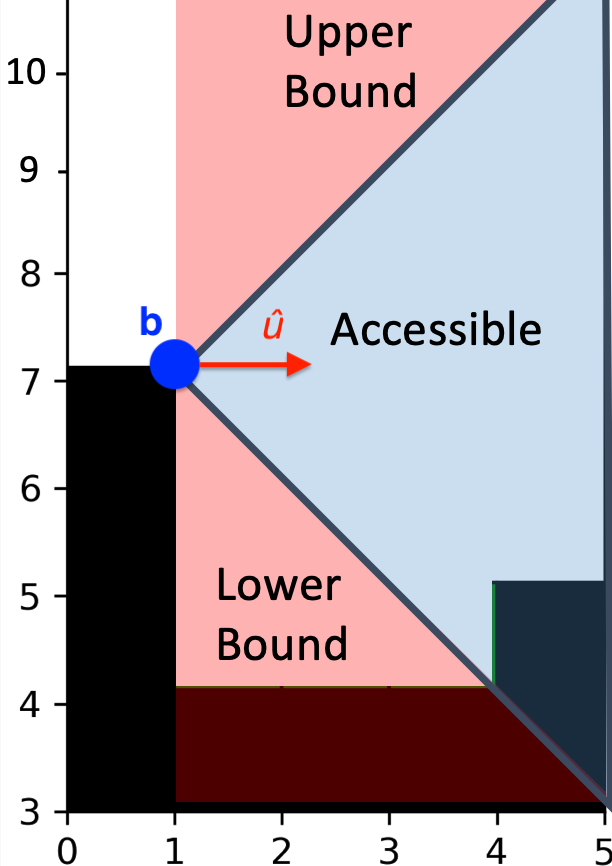}
 \includegraphics[width=0.32\columnwidth]{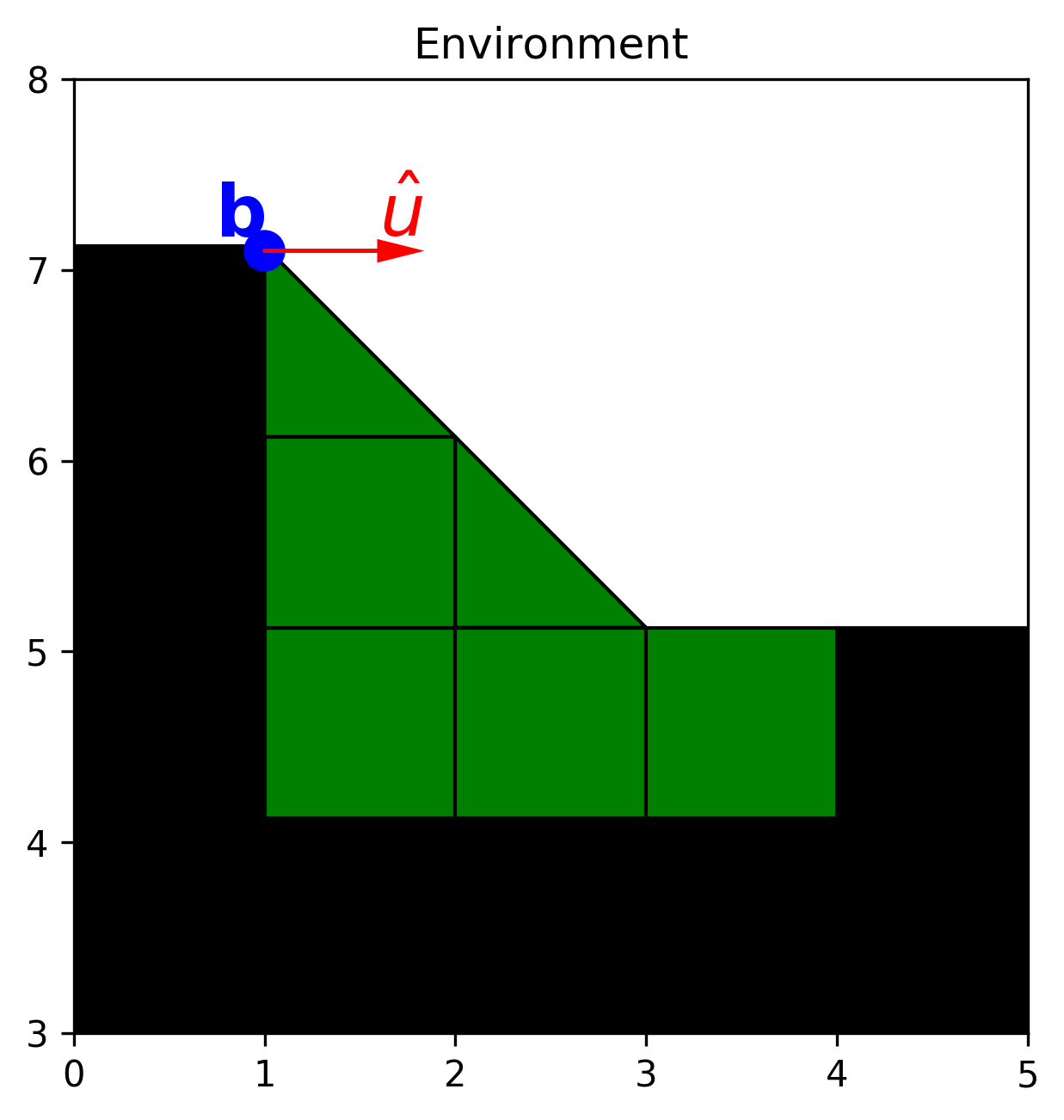}
 \caption{Detailed view of the structure synthesized by our algorithm to connect the top of the chair (z=560) to the top of the table (z=400). \textbf{Left:} Top-down (X-Y) view showing structure cells.  \textbf{Center:} Side ($\strVector$-Z) view showing upper and lower bounds used in the Waterfall algorithm. \textbf{Right:} Side view showing structure synthesized by Waterfall.}
\vspace{-0.5cm}
 \label{fig:example-environment}
\end{figure}
%
%
\section{Waterfall Algorithm - Generating All Useful Structures}
\label{sec:waterfall}
To solve this problem, our first step is to generate a set $\structures$ of potentially useful structures.
A valid structure may be placed at any position and orientation in the environment, but we observe that  it is only \emph{useful} if it connects two different regions, so we ought to look for structures that bridge the boundaries between regions. Regions are bounded by sloped areas (where nodes are removed) and sharp cliffs (where edges are removed), but because sloped areas violate the conditions for buildability, we consider only cliffs as candidate structure locations.  We define the set of \emph{build points} $\buildpoints = \enumSet{\buildpoint}{n}$ to be the coordinates of the cliff edges.  At each build point, we attempt to synthesize a structure connecting the top of the cliff to another region.

For the robot to move from the cliff surface onto the structure, the structure's orientation $\strVector$ should be nearly perpendicular to the cliff surface (Definition~\ref{def:valid-structure}). In our implementation, the cliff surface normal at $\buildpoint$ is estimated by selecting all nodes with a radius of $\lblock$ of $\buildpoint$ and training a linear classifier (specifically an SVM) using node coordinates as features and node regions as labels; the resulting classification plane provides a suitable estimate of the surface normal. The error rate of the classifier provides a measure of surface flatness: if more than 10\% of training points are incorrectly classified, the boundary is considered non-flat and the build point is rejected.  

\subsection{Algorithm}
Given build point $\buildpoint$ and orientation $\strVector$, we synthesize valid structures by solving the constraint satisfaction problem imposed by the conditions for structure validity (Definition~\ref{def:valid-structure}).  
Using parameters $\lblock=8cm$ and $\dzcliff=4cm=\lblock/2$ for our robot and building blocks, 
the constraint imposed by Equations~\ref{eqn:Zi}-\ref{eqn:traversable} imposes upper and lower bounds on the cell heights $\height_i$ (measured in number of blocks) as a function of the bounds on its neighbors:
%
%
\begin{eqnarray}
  \cellmax(i)_j =&  \left\lfloor \cellmax(j) + \frac{\environment(\cell_{j})-\environment(\cell_i)}{\lblock}+1.5 \right\rfloor \\ 
  \cellmin(i)_j =&  \left\lceil \cellmin(j) + \frac{\environment(\cell_{j})-\environment(\cell_i)}{\lblock}-1.5 \right\rceil 
\end{eqnarray}
The brackets $\lceil x \rceil$ and $\lfloor x \rfloor$ denote the ceiling and floor functions, which round $x$ up or down to the nearest integer (respectively).
The above equations enforce the constraints imposed on $\cell_i$ by $\cell_{j}$, where $j=i \pm 1$.
Fig. \ref{fig:example-environment} illustrates these upper and lower bounds for an example structure.  

The $\waterfall$ algorithm (Algorithm~\ref{alg:waterfall}) solves for structures by propagating these bounds outwards from $\buildpoint$.  It begins by placing the first structure cell at the top of the cliff, assigning it height $h=0$. It then marches cells into the lower region, outwards from $\buildpoint$ in the direction of $\strVector$ (Fig.~\ref{fig:example-environment}).  For each new cell, it calculates $\height_{min}$ and $\height_{max}$, propagating the constraints forward. If it encounters a cell that is unbuildable, or for which $\height_{max}<0$, it returns False (no structure can be built here). If it finds a cell for which $\height_{min}<0$, it records this cell as the endpoint of the structure, and assigns it height $\height=0$.

The algorithm then marches back to $\buildpoint$, assigning cell heights.  At each cell $c_i$, it checks the $\height_{min}$ constraint in the backwards direction (imposed on $i$ by $i+1$), and assigns $\height_i$ to be the maximum of the lower-bounds from enforcing the constraint in both directions.

The algorithm performs a  final pass to assign terminators.  At each cell $i$, it considers the difference in height $Z$ between its neighbors $i+1$ and $i-1$.  If magnitude of the difference is less than $\lblock$/2, the terminator is a block (flat); otherwise, the terminator is a wedge oriented so that its slope matches the sign of $\Delta Z$.
\subsection{Correctness Guarantees}

\begin{lemma} [Optimality]\label{lem:waterfallopt}
$\waterfall$ finds the cheapest valid structure (if it exists) that can be built at  $\buildpoint$ which will connect the region at the top of the cliff to another region.
\end{lemma}
\begin{proof}
Valid structures must begin and end with zero-height cells (i.e. at the environment surface).  In its first pass, $\waterfall$ finds the closest cell to $\buildpoint$ at which a valid structure could terminate, and selects this cell as the end of the structure.
At each cell between the terminators, $\waterfall$ computes the lower bound on structure height imposed by propagating the traversability constraints forward from the start cell and backwards from the end cell, and selects the minimum structure height satisfying both constraints.
\end{proof}

\begin{lemma} [Completeness]\label{lem:waterfallcomplete}
For the purposes of minimal structure construction, 
we need not consider any structures that could be built at $\buildpoint$ other than the one produced by  $\waterfall$.
\end{lemma}
\begin{proof}
To show this, consider if it were not the case. 
Say a longer structure $\structure_L$ (connecting $\buildpoint$ to a different region) is required as part of the minimum spanning tree of structures. This structure would cut through multiple regions.
Based on the buildability conditions, at each point on a region boundary that this structure crosses through there is another potential build point, with its own potential minimal structure identified by the $\waterfall$ algorithm in the $\strVector$ direction. All of these minimal structures must cost less than $\structure_L$, and their footprints are fully contained within the footprint of $\structure_L$. Call the set of these structures $\structures_L$.

Now consider if $\structure_L$ were removed from the minimum spanning tree. The environment would then be separated into two separate components. At least one of the minimal structures in $\structures_L$ must bridge these components. If this structure is added back to the minimum spanning tree, then the environment is again fully traversable, and the new set of structures costs less than the original. So the original spanning tree of structures was not minimal.
We therefore only need to search through the set of structures discovered through $\waterfall$.
\end{proof}
\subsection{Runtime and Generalization}
The structure synthesis problem can be formulated as an integer linear program (ILP) -- the variables (cell heights) are integer valued, while the constraints (traversability) and objective (number of blocks used) are linear.  Well-established algorithms can solve general ILP problems, however ILP is known to be NP-Complete.  For the particular case of our building blocks, we are able to solve this ILP efficiently using the $\waterfall$ algorithm, which synthesizes structures with $N$ cells in $O(N)$ time.  For other kinds of building blocks, the $\waterfall$ algorithm could be modified, or a general ILP solver could be used.
%
\newcommand{\hx}{hx}
\newcommand{\hn}{hn}
\newcommand{\h}{h}
\begin{algorithm}
  \begin{algorithmic}[1]
    \Require{build point $\buildpoint$ and direction $\strVector$}
    \Statex $\hn$, $\hx$, $\h$, and $t$ are lists that resize automatically.
    \Function{Waterfall}{$\buildpoint, \strVector$}
    \State $end \gets \infty \mrm{,~} \hx[0] \gets 0 \mrm{,~} \hn[0] \gets 0$
    \For{$i \gets 1;~ i< \infty;~ i \gets i + 1$}  \Comment{Pass 1: Endpoint}
      \If{$\neg buildable(i)$}
        \Return False
      \EndIf
      \Let{$\hn[i]$}{$\cellmin(i)_{i-1}$}
      \Let{$\hx[i]$}{$\cellmax(i)_{i-1}$}
      \If{$\hx[i] < 0$}
        \Return False
      \EndIf
      \If{$\hn[i] < 0$}
        \State $end \gets i \mrm{,~~} \hx[i] \gets 0 \mrm{,~~} \hn[i] \gets 0$
        \State \textbf{break}
      \EndIf
    \EndFor
    \State $h[0] \gets 0 \mrm{,~~} h[end] \gets 0$ \Comment{Pass 2: Heights}
    \For{$i \gets end-1;~ i>0;~ i\gets i -1$}
      \State $\hn[i] \gets \max (\cellmin(i)_{i+1}, ~\hn[i] )$
      \Let{$\h[i]$}{$\hn[i]$}
    \EndFor
    \State $t[0] \gets none~;~ t[end] \gets none$ \Comment{Pass 3: Terminators}
    \For{$i \gets 1;~ i< end;~ i \gets i+1$}
      \Let{$\Delta Z$}{$\environment(\cell_{i-1})-\environment(\cell_{i+1})+\lblock (\h[i-1])-\h[i+1])$}
      \If{$\left\lvert \Delta Z \right\rvert \le 0.5*\lblock$}
        $t[i] \gets block$
      \ElsIf{$\Delta Z < 0$}
        $t[i] \gets \fwedge$
      \ElsIf{$\Delta Z > 0$}
        $t[i] \gets \bwedge$
      \EndIf
    \EndFor
    \State \Return $\langle h,t \rangle$
    \EndFunction
  \end{algorithmic}
  \caption{Waterfall Algorithm}
  \label{alg:waterfall}
\end{algorithm}
\section{BB-MST Algorithm -- Solving for the Minimum Spanning Tree of Structures}
\label{sec:bbmst}
Once the useful structures are identified, then these structures can be used to make the environment traversable. 
Given the set  $\structures$ of structures generated by running $\waterfall$ at every build point, form graph $\regionGraph(\regions, \structures)$ with nodes as regions and edges as structures connecting regions, and assign edge weights equal to the cost of their structures. We seek $\solution{\regions}{\structures}$, a conflict-free min-cost spanning tree of $\regionGraph(\region, \structures)$.   
We refer to this problem as \emph{struct-MST}.

\subsection{NP-Hardness of struct-MST}
Kruskal's algorithm computes minimum-spanning trees of graphs in $O(|E| \log |V|)$ time.  However, structure-structure and region-structure conflicts (Section~\ref{sec:conflicts}) impose constraints that make struct-MST a much more difficult problem to solve efficiently. Structure-structure conflicts create \emph{pairwise negative disjunctive constraints} between edges in $\regionGraph$, that is, pairs of edges that cannot both be present in the solution.  These constraints may be represented in terms of a \emph{conflict graph} with vertices as edges in the original graph, and edges as constraints.  It has been shown that deciding the existence of a spanning tree (as well as finding the min-cost spanning tree) of a graph is strongly $\mathcal{NP}$-hard under negative disjunctive constraints, unless the conflict graph has a maximum path length less than two \cite{darmann2011paths}. This is not the case for struct-MST (because any structure could conflict with multiple others), making struct-MST at least $\mathcal{NP}$-hard. 
\subsection{Algorithm}
We present $\bbmst$, a branch-and-bound algorithm to solve struct-MST.  Branch-and-bound is a general algorithm for solving combinatorial optimization problems, and has been employed as a practical technique to solve many $\mathcal{NP}$-hard problems \cite{lawler1966branch}.

Algorithm~\ref{alg:BB-MST} provides pseudocode for $\bbmst$.
Each time $\bbmst$ is called, it uses Kruskal's algorithm to solve for $\mst_{\regions,\structures}$, the MST of the regions and structures passed in as arguments, and compares it to the best-yet solution $\mst^*$. Regardless of whether $\mst_{\regions, \structures}$ is a valid solution, if $\cost(\mst_{\regions,\structures}) \ge \cost(\mst^*)$, $\bbmst$ returns $\mst^*$, terminating search of the branch.
If $\mst_{\regions,\structures}$ is conflict-free and $\cost(\mst_{\regions,\structures}) < \cost(\mst^*)$, $\mst^*$ is updated and $\mst_{\regions,\structures}$ is returned. 

If $\mst_{\regions,\structures}$ includes a conflict, $\bbmst$ recursively branches, solving two or more child problems in which some of the conflicting edges or nodes have been removed, and returns the cheapest of the child solutions.   Structure-structure and region-structure conflicts are handled as follows:

\subsubsection{Structure-Structure Conflicts}
If there is a conflicting pair of structures $\{\structure_1, \structure_2\}$ in $\mst_{\regions, \structures}$, form two child problems, removing one conflicting structure from each: $\mst_1=\bbmst(\regions, \structures \setminus \structure_1)$ and $\mst_2=\bbmst(\regions, \structures \setminus \structure_2)$.  Return the cheaper of the two solutions.
\subsubsection{Region-Structure Conflicts} \label{sec:region-structure-conflict}
Region-structure conflicts are somewhat more complex.  Let $\structures(\mst_{\regions,\structures})$ be the set of edges of $\mst_{\regions, \structures}$.  $\forall (\structure, \region) \in \structures(\mst_{\regions,\structures}) \times \regions$, check whether $\structure$ has more than one structure cell containing a boundary node of $\region$; if so, $\structure$ might split $\region$. 
Let $\region_c$ be one such region, and $\structures_c = \enumSet{\structure}{k}$ be the set of structures meeting this condition.  Let $V_{\structures_c}$ be the set of nodes in the grid-graph $G(V,E)$ occupied by $\structures_c$, $V_{\region_c}$ be the set of nodes in $\region_c$, and $G_{\region_c}$ be the subgraph of $G(V,E)$ induced by $V_{\region_c} \setminus V_{\structures_c}$. Compute $\regions_{\region_c}$, the set of connected components of $G_{\region_c}$: if $R_{r_c}$ has more than one element, $\region_c$ has been split and is in conflict with $\structures_c$.

We handle this by forming $k+1$ branches, where $k$ is the number of structures in $\structures_c$. In each of the first $k$ branches, we remove one conflicting structure: $\mst_i = \bbmst(\regions,\structures \setminus \structure_i) ~\forall~ \structure_i \in \structures_c$.  In the final branch, we keep all conflicting structures,  but split the region $\region_c$ into multiple sub-regions: Letting $\regions_{split} = (\regions \setminus \region_c) \cup \regions_{\region_c}$, we have $\mst_{split} = \bbmst(\regions_{split}, \structures)$.  We return the cheapest of all solutions $\left\lbrace \mst_1,\mst_2, \ldots, \mst_k, \mst_{split}\right\rbrace$.
\newcommand{\branchEdges}{\operatorname{\textsc{branchEdges}}}
\newcommand{\branchRegion}{\operatorname{\textsc{branchRegion}}}
\begin{algorithm}
  \begin{algorithmic}[1]
    \Require{Regions $\regions$ and candidate structures $\structures$}
    \Statex
    \State Initialize $\mst^* \gets \emptyset$ 
    \Function{BB-MST}{$\regions, \structures$}
      \State Form $\regionGraph(\regions, \structures)$
      \If{$\regionGraph$ is not a connected graph}
        \Return $\emptyset$
      \EndIf
      \Let{$\mst_{\regions, \structures}$}{$kruskal(\regionGraph)$} \label{line:kruskal}
      \If{$\cost(\mst_{\regions, \structures}) \ge \cost(\mst^*)$}
        \Return $\mst^*$
      \EndIf
      \If{$\cost(\mst_{\regions, \structures}) < \cost(\mst^*)$ and $\mst_{\regions, \structures}$ is valid}
        \Let{$\mst^*$}{$\mst_{\regions, \structures}$}
        \State \Return $\mst_{\regions, \structures}$
      \EndIf
      \If{$\mst_{\regions, \structures}$ has a structure conflict $\lbrace \structure_1, \structure_2 \rbrace$}
        \State \Return $\branchEdges(\lbrace \structure_1, \structure_2 \rbrace,\regions, \structures)$
      \ElsIf{$\mst$ has a region conflict $\lbrace \region_c, \structures_c \rbrace$}
        \Let{$\mst_{edges}$}{$\branchEdges(\structures_c, \regions, \structures)$} 
        \Let{$\mst_{split}$}{$\branchRegion(\region_c,\structures_c,\regions,\structures)$} 
        \State \Return the cheaper of $\mst_{edges},~ \mst_{split}$ 
      \EndIf
    \EndFunction
    \Statex 
    \Function{branchEdges}{$\structures_c, \regions, \structures$}
      \Let{$\mst$}{$\emptyset$}
      \ForAll{$\structure \in \structures_c$}
        \Let{$\mst_\structure$}{$\bbmst(\regions, \structures \setminus \structure)$}
        \If{$cost(\mst_\structure)<cost(\mst)$}
          $\mst \gets \mst_\structure$
        \EndIf
      \EndFor
      \State \Return $\mst$
    \EndFunction
    \Statex 
    \Function{branchRegion}{$\region_c, \structures_c, \regions, \structures$}
      \Let{$\regions_{\region_c}$}{$\mathrm{split\_into\_subregions}(\region_c, \structures_c)$}
      \Let{$\regions_{split}$}{$(\regions \setminus \region_c )\cup \regions_{\region_c}$}
      \ForAll{$\structure \in \structures$}
        \If{$\structure$ connected to $\region_C$}
          \State Reassign $\structure$ to connect to appropriate $\region \in \regions_{split}$
        \EndIf
      \EndFor
      \State \Return $\bbmst(\regions_{split}, \structures)$
    \EndFunction
  \end{algorithmic}
  \caption{BB-MST Algorithm}
  \label{alg:BB-MST}
\end{algorithm}
\subsection{Correctness Guarantees}
\begin{lemma}[Child Bounding] \label{lem:bounding}
  Let $\mst_{\regions, \structures}$ (returned by Kruskal's algorithm, Algorithm~\ref{alg:BB-MST} line~\ref{line:kruskal}) contain one or more conflicts. $\cost(\mst_{\regions, \structures})$ is a lower bound on the cost of the solution to any child problem formed by branching on a conflict in $\mst_{\regions, \structures}$.
\end{lemma}
\begin{proof}
  Whenever $\bbmst$ branches, it either eliminates one structure, or it splits one region into multiple regions. 
  Consider the case where structure $\structure$ has been eliminated. It is clear that $\cost(\mst_{\regions, \structures}) \le \cost(\mst_{\regions, \structures \setminus \structure})$: $\mst_{\regions, \structures \setminus \structure}$ is optimal with respect to $\structures \setminus \structure$, so making $\structure$ available could only decrease cost.
  Consider the case where region $\region_c \in \regions$ has been split into two sub-regions $\region_1, \region_2$, resulting in a new set of regions $\regions_{split}=(\regions \setminus \region_c) \cup \{\region_1, \region_2\}$. 
$\mst_{\regions_{split}, \structures}$ is the MST which spans $\regions_{split}$. Form a new graph $\mathcal{\mst}$ identical to $\mst_{\regions_{split}, \structures}$ except that nodes $\region_1, \region_2$ have been merged together to form node $\region_c$. We may remove one edge from this graph to form a tree, which we will denote $\mst'$.  By construction, $\cost(\mst_{\regions_{split}, \structures}) = \cost(\mathcal{\mst}) \ge \cost(\mst')$. $\mst'$ spans the same set of nodes as $\mst_{\regions, \structures}$, but $\mst_{\regions, \structures}$ is the MST, so $\cost(\mst_{\regions, \structures}) \le \cost(\mst') \le \cost(\mst_{\regions_{split}, \structures})$. 
\end{proof}

\begin{lemma}\label{lem:anytime}
$\bbmst(\regions,\structures)$ returns the best-yet set of spanning structures that connects the regions in the environment.
\end{lemma}
\begin{proof}
We prove by induction that each recursive call of $\bbmst(\regions, \structures)$ returns either $\mst^*_{\regions, \structures}$ or $\mst^*$ (the best-yet solution), whichever is cheaper.
\subsubsection{Base Case}
There are three conditions under which $\bbmst$ returns without branching. \textbf{(1) Null:} $\regionGraph$ is disconnected, so no spanning tree can be found and $\bbmst$ returns $\emptyset$. \textbf{(2) Shortcut:} $\cost(\mst_{\regions, \structures}) \ge \cost(\mst^*)$, so $\bbmst$ returns $\mst^*$; by Lemma~\ref{lem:bounding}, this branch cannot contain a solution cheaper than the current best-yet solution $\mst^*$, so we stop exploring it. \textbf{(3) Success:} $\cost(\mst_{\regions, \structures}) < \cost(\mst^*)$ and $\mst_{\regions, \structures}$ is conflict-free, so we set $\mst^* \gets \mst_{\regions, \structures}$ and return $\mst_{\regions, \structures}$.  In this case, it is clear that $\mst_{\regions, \structures} = \solution{\regions}{\structures}$ since Kruskal's algorithm produces a min-cost spanning tree, which is explicitly verified as conflict-free.
\subsubsection{Induction Step}
When $\mst_{\regions, \structures}$ has a conflict, $\bbmst$ forms two or more child branches and returns the cheapest solution among them. Assuming recursive calls of $\bbmst(\regions, \structures)$ return $\solution{\regions}{\structures}$ for the set of regions and structures passed down to them, we prove that the optimal solution to the parent problem must be the cheapest of its children. 

Let $\mst_{\regions,\structures}$ contain structure-structure conflict set $\{\structure_1,\structure_2\}$. Since $\{\structure_1,\structure_2\}$ cannot be in the solution, we know $\solution{\regions}{\structures}$ must be a subset of either $\structures \setminus \structure_1$ or $\structures \setminus \structure_2$, since $\powerset(\structures \setminus \structure_1) \bigcup \powerset(\structures \setminus \structure_2)$ is the set of all subsets of $\structures$ that do not contain $\{\structure_1, \structure_2\}$. Therefore, $\solution{\regions}{\structures}$ must be the cheaper of $\solution{\regions}{\structures
\setminus \structure_1}$ and $\solution{\regions}{\structures \setminus \structure_2}$.

By similar reasoning, let $\mst_{\regions,\structures}$ contain a region-structure conflict set $\{ \region_c, \structures_c \}$. We know that either $\solution{\regions}{\structures} = \solution{\regions}{\structures \setminus \structure_c}$ for some $\structure_c \in \structures_c$, or $\solution{\regions}{\structures} = \solution{\regions_{split}}{\structures}$, since in each of these cases a single member of the conflict set has been removed.  $\bbmst$ returns the cheapest of these options by comparing $\mst_{edges}$ and $\mst_{split}$, returned by $\branchEdges$ and $\branchRegion$.
\end{proof}

\begin{theorem} \label{thm:main-result}
A combination of $\bbmst$ and $\waterfall$ generates the minimum cost set of structures that makes the environment fully traversable.
\end{theorem}
\begin{proof}
According to Lemma~\ref{lem:waterfallopt}, the set of structures produced using the $\waterfall$ algorithm contains all the minimal structures needed to connect each region to its neighbors. Lemma~\ref{lem:waterfallcomplete} states that no other structures are needed to produce a minimal spanning set. The $\bbmst$ algorithm searches through all possible subsets of structures and, by Lemma~\ref{lem:anytime}, outputs the best-yet spanning set. Therefore, if the algorithm $\bbmst$ is allowed to run to completion using the full set of structures provided by $\waterfall$, it will produce the minimum cost spanning set.
\end{proof}
\subsection{Runtime}
$\bbmst$ resolves conflicts by recursively exploring each possible conflict-free subset of the conflict set, so in the worst case it will explore an exponential number of branches before finding a solution.  This is to be expected: struct-MST is NP-Hard, and $\bbmst$ solves the problem exactly.

In practice, $\bbmst$ typically prunes many branches and explores a tiny fraction of this space. 
Additionally, once a feasible solution is found, $\bbmst$ has an \textit{anytime} property: it can be terminated at any time and return the best-yet feasible solution.
To identify that no conflict-free spanning solutions exists, the algorithm must explore each branch until $\regionGraph$ becomes disconnected, which can be very time-consuming. 

It is worth noting that relaxing the optimality requirement of struct-MST would not improve the worst-case runtime, because deciding the existence of a (non-minimal) spanning tree of structures is also $\mathcal{NP}$-hard.
\section{Results}
\label{sec:results}
\subsection{Examples and Experiments}
Our Python implementation can solve for optimal sets of structures for the SMORES-EP robot and building blocks
given a height-field representation of an environment.  The implementation accepts CAD models and 3D maps of real environments as inputs, and we show that it can solve for optimal sets of structures in real-world indoor environments.
\subsubsection{CAD Example -- Table-and-Chair}
Figure \ref{fig:table-chair-optimal} shows optimal solutions for two configurations of the table-and-chair CAD example environment  from Figure~\ref{fig:example-environment}.  In Example \exTableChair~-``Chair Pulled Out'', the optimal solution uses one large ramp from the floor to the chair seat, and two smaller structures connecting the seat to the tabletop and tabletop to chair top.  Notice the position the structure connecting the tabletop to the chair top: the structure crosses over the chair seat, and has been placed at the far left edge of the chair seat to avoid creating a region-structure conflict (which would have required a fourth structure).  In Example B, the chair has been pushed in further, and there is no longer enough space to place structures on the chair seat.  The algorithm is forced to select a more expensive solution using three large ramps from the floor.
\begin{figure}
\centering        
  \renewcommand{\arraystretch}{0.2}
  \footnotesize
  \begin{tabular}{|c|c|}
    \hline \\
    \multicolumn{2}{|c|}{\exTableChair~Chair Pulled Out} \\
    \multicolumn{2}{|c|}{ 
    \includegraphics[width=0.35\columnwidth]{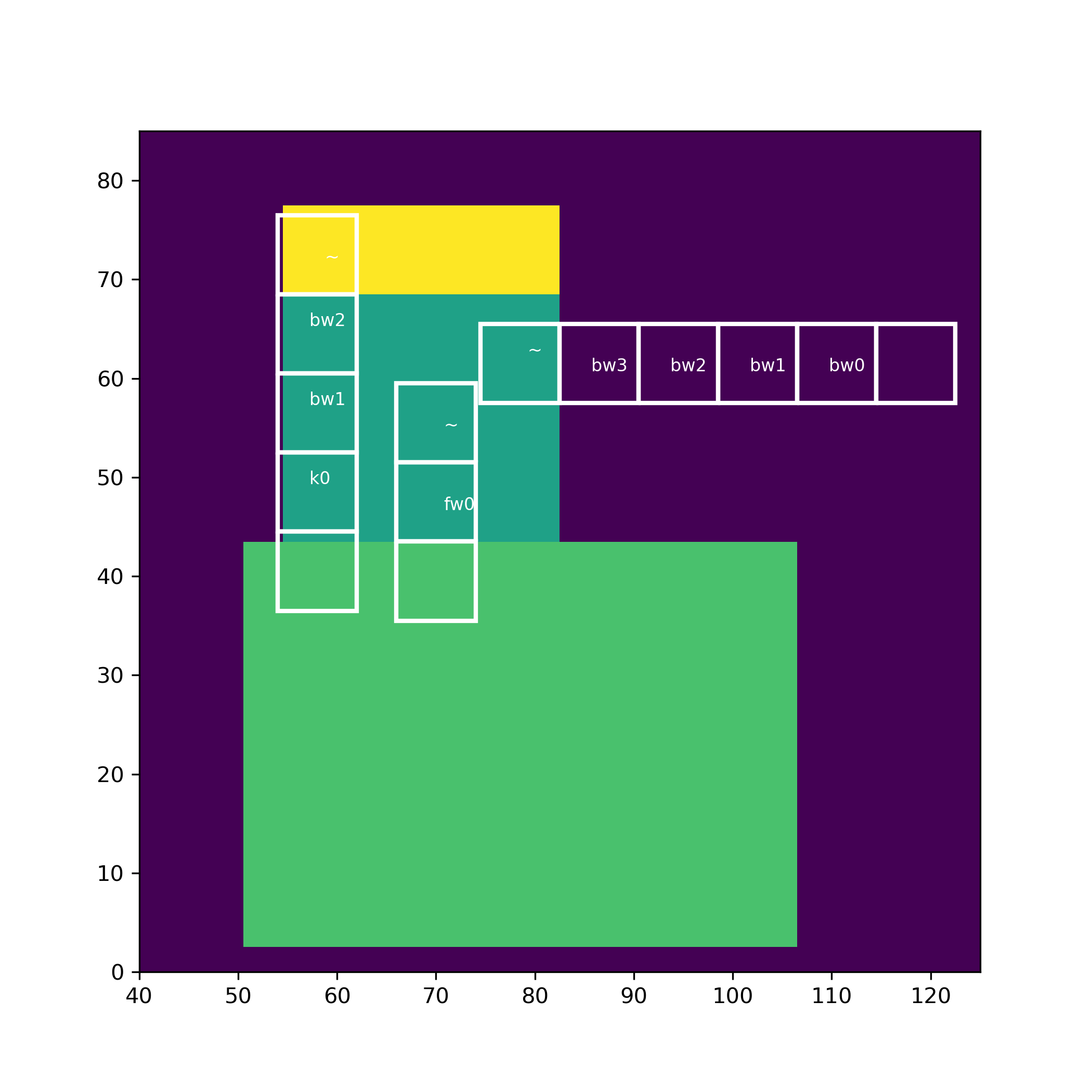} 
    \includegraphics[width=0.35\columnwidth]{images/tc-optimal-3d-2.png} } \\
    Height Field & 3D Rendering \\
    \hline \\
    \exPushedIn~Chair Pushed In & \exCheckerboard~Random checkerboard \\ 
    \includegraphics[width=0.35\columnwidth]{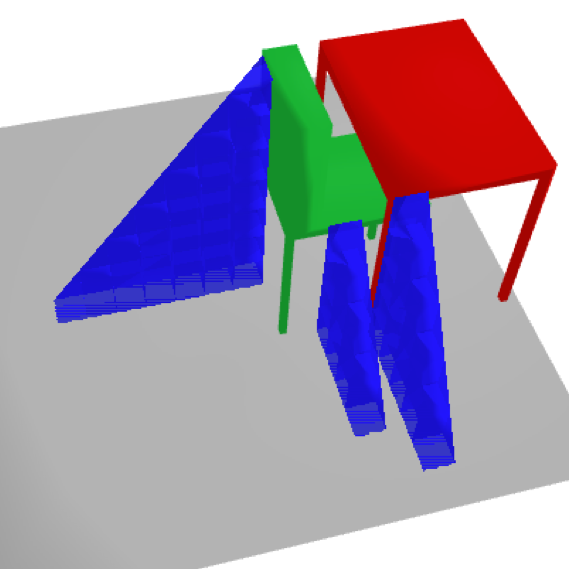} & 
    \includegraphics[width=0.35\columnwidth]{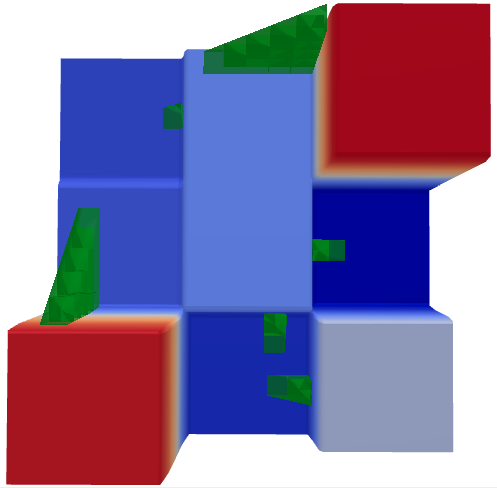} \\
    \hline
  \end{tabular}
  \caption{Optimal solutions in simulated environments.}
  \label{fig:table-chair-optimal}
  \vspace{-0.5cm}
\end{figure}
\subsubsection{Real-World Experiments}
Figure~\ref{fig:real-world} shows solutions scanned from two indoor environments.  The video accompanying this paper shows the SMORES-EP robot moving through these environments using structures placed in the locations selected by the algorithm.

To solve for structures in these environments,  3D occupancy grid maps were created with the Octomap library \cite{hornung2013octomap} using point cloud data scanned with a Microsoft Kinect RGB-D sensor.  A similar RGB-D sensor can be carried by SMORES-EP to autonomously map its environment \cite{JD:GJ:TT:MY:HK:MC:18}.  Occupancy grids were converted to height fields (2D grayscale images), which were smoothed (median filter, window of 3 pixels) and segmented (K-means, K=150) to reduce noise before running the algorithm to generate structures.  Real environments often contain regions that are too small for the robot to occupy them, even if it could access them (for example, the arms and back of the chair in Figure~\ref{fig:real-world}).  To account for this, we test whether each region can fit an inscribed 8cm square (the size of one SMORES-EP robot module) anywhere within it, and remove small regions from the set of regions $\regions$ before running the algorithm.

In Figure~\ref{fig:real-world}, Environment \exOffice~consists of a round table, a stack of magazines, an office chair, and a storage container.  The algorithm determines that adding one 4-block high ramp structure and three single wedges to this environment will allow the robot to reach every surface large enough to support it.  Environment \exStairs~is a staircase.  The algorithm determines that six 2-block high ramps can be introduced to make it globally traversable by the robot. 
The location of each ramp on a given step is effectively random, because solutions that use the same set of structures in different locations have equal cost.  
\begin{figure}
\centering
  \renewcommand{\arraystretch}{0.2}
  \newcommand{\fh}{1.2in}
  \footnotesize
  \begin{tabular}{|c c|}
    \hline \\
    \multicolumn{2}{|c|}{\exOffice~Table, Chair, and Box} \\
    \includegraphics[height=\fh]{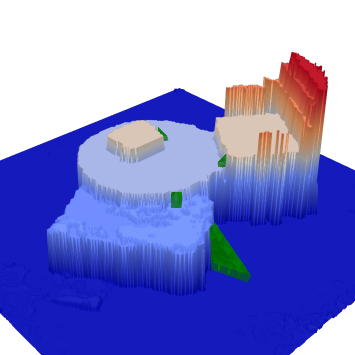} &
    \includegraphics[height=\fh]{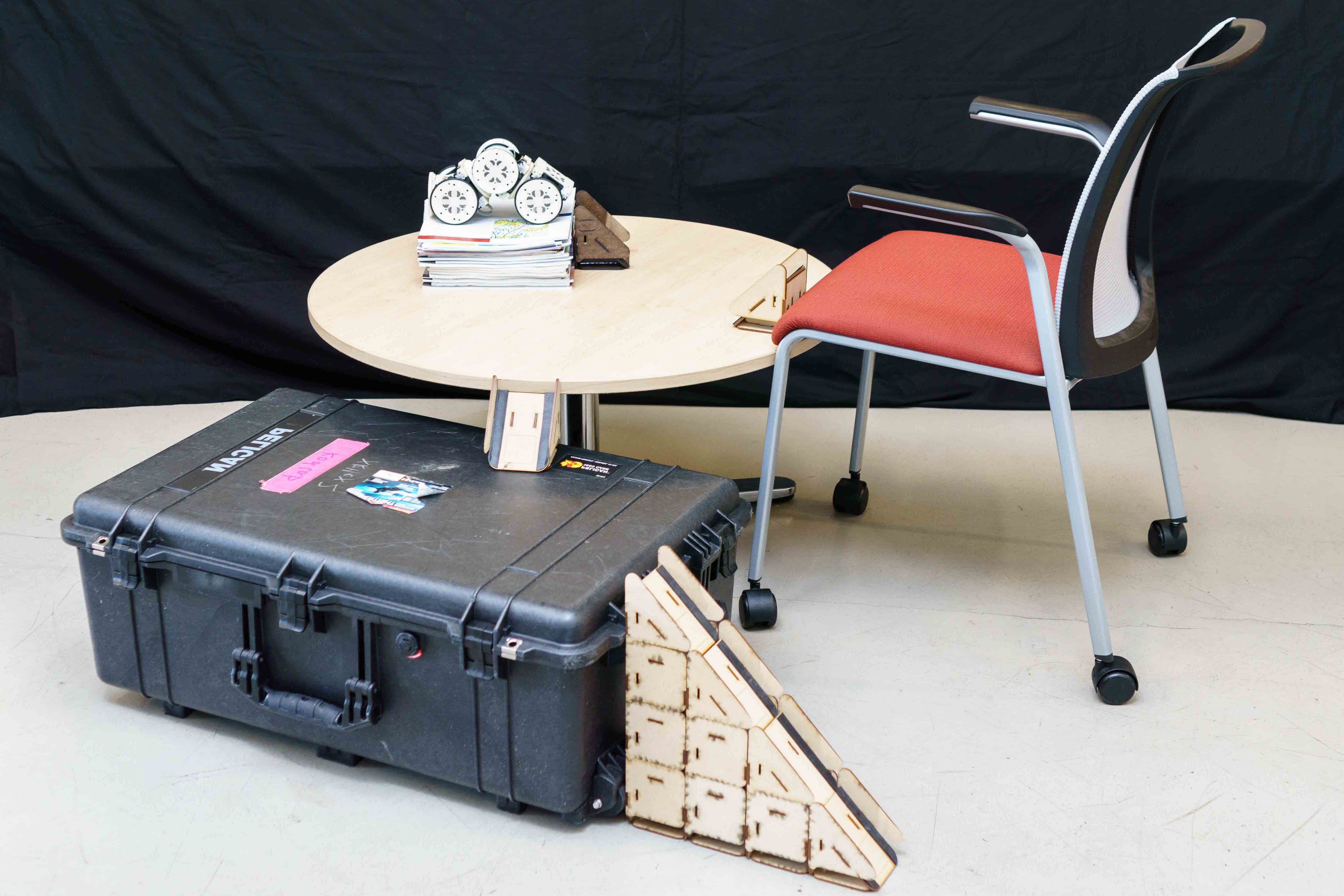} \\
    \hline \\
    \multicolumn{2}{|c|}{\exStairs~Stairs} \\
    \includegraphics[height=\fh]{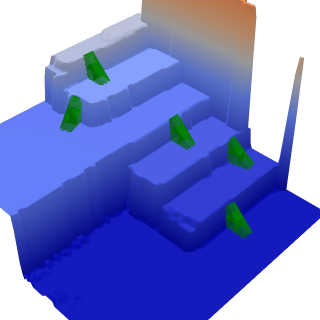} &
    \includegraphics[height=\fh]{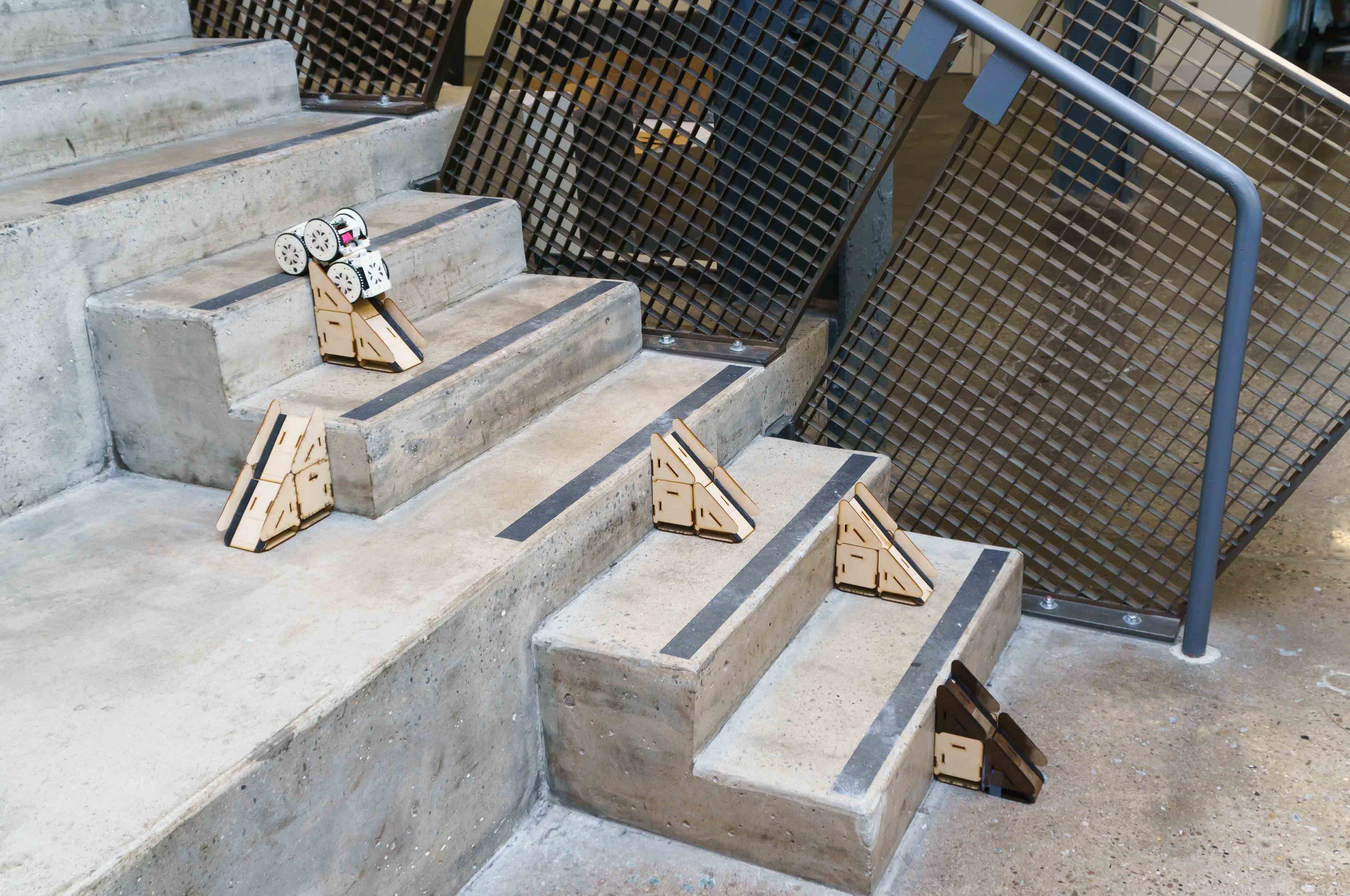}\\
    \hline
  \end{tabular}
  \caption{Real-world environments, and algorithm solutions generated from 3D map data taken with a Kinect sensor.}
  \label{fig:real-world}
  \vspace{-0.5cm}
\end{figure}
\subsection{Runtime Performance}

Table~\ref{tab:runtime} shows metrics for the example environments from Figs~\ref{fig:table-chair-optimal} and \ref{fig:real-world}.
The algorithm generates thousands of structures and solves problems with 4-9 regions in minutes. In many cases, the number of potential conflicts (which determines the worst-case runtime) is in the tens of thousands, but the algorithm explores a tiny fraction of them (less than thirty).  The Conflict Pair Fraction (CPF) for each problem is the percentage of pairs of structures which conflict, and provides a measure of problem difficulty. In all examples, a significant fraction of the total time required to reach a solution is spent preprocessing the world (e.g. generating the initial grid graph and identifying cliffs edges and steep areas).

Runtime performance was profiled by generating and solving random environments similar to environment \exCheckerboard~(Fig~\ref{fig:table-chair-optimal}). Each environment is a $3\times3$ checkerboard with $6 \lblock$-wide squares with randomly selected heights between zero and $3 \lblock$ (in one-pixel increments).  Trials were terminated after a timeout of 10 minutes.
Of 1118 total trials, 981 found a solution (conflict-free MST), and 137 found that no solution exists.  834 solutions were found in one try, and 16 no-solutions were found in one try.  17 trials timed out, of which 11 found no solution and 6 found a feasible solution.  On average, each problem generated 363 structures with 10184 conflicts, and CPF of 7.83\%.

%
\begin{figure}
  \includegraphics[width=\columnwidth]{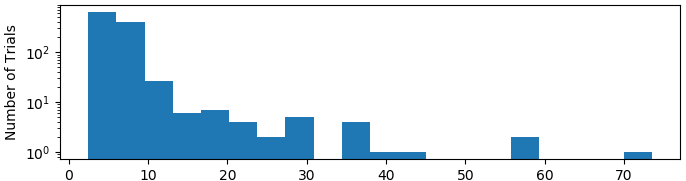}
  \vspace{-0.5cm}
  \caption{Log-histogram of solution times for 1118 random environments.  In addition to the data shown, 17 environments timed out after 10 minutes.}
  \label{fig:time-histogram}
  \vspace{-0.3cm}
\end{figure}
%
%
\begin{table}
\footnotesize
\input{runtime_table.tex}
\vspace{-0.6cm}
\caption{Runtime Performance\vspace{-0.6cm}}
\label{tab:runtime}
\end{table}
\vspace{-0.2cm}
\section{Discussion and Future Work}
\label{sec:discussion}
\vspace{-0.1cm}
This paper presents an algorithm to find min-cost spanning sets of structures allowing a robot to reach every surface of an environment.  Finding optimal (as opposed to feasible) solutions to these problems is important - building larger structures takes more time, and in real scenarios the number of available building blocks is always limited.  For example, the solution in environment \exTableChair~requires 17 blocks, whereas solution \exPushedIn~(which is also a feasible solution for \exTableChair) requires 46. 
%

Given a build point and direction vector, $\waterfall$ generates an optimal structure in linear time, allowing thousands of candidate structures to be generated in under a minute.  Given a set of candidate structures, $\bbmst$ will always eventually find a conflict-free solution to a struct-MST problem, if one exists.  Because struct-MST is NP-Hard, any algorithm that solves it will have exponential worst-case complexity, and for some problems (especially when no solution exists) $\bbmst$ will run for an impractically long time before returning.  However, in typical problems, $\bbmst$ explores the solution space efficiently and returns optimal solutions in a few seconds. In many realistic problems, the number of potential conflicts is relatively small compared to the number of potentially useful combinations of structures. For example, the \exOffice~ and \exStairs~ have low CPF values, and in both cases the first MST generated had no conflicts. 

Some tasks might require a robot to access only a subset of the regions in an environment. The framework introduced in this paper could easily be extended  to solve for min-cost \textit{paths} of structures (connecting a pair of regions) by calling Dijkstra's algorithm in place of Kruskal's algorithm in $\bbmst$.  With slight modification, the framework could also solve for approximately-optimal \textit{Steiner trees} of structures, to make a selected subset of regions accessible.  Solving for min-cost Steiner trees in graphs is $\mathcal{NP}$-hard, but poly-time algorithms can solve the problem approximately \cite{byrka2010improved}.  Calling such an algorithm in place of Kruskal's algorithm would allow $\bbmst$ to compute Steiner trees, as long as the approximation factor is taken into account when comparing solution costs in the shortcut-return case.

Future work includes taking robot path planning into account when selecting structure locations. For example, in the stairs environment in Figure~\ref{fig:real-world}, placing the ramps in a line would allow the robot to move more efficiently through the environment.  Optimization of structure positions could be performed as a separate post-processing step after the algorithm selects structures, or information about structure position could be directly incorporated into the cost function for evaluating solutions.  
\vspace{-0.2cm}
\section{Conclusion}
\vspace{-0.1cm}
This paper presents an complete, optimal algorithm to generate sets of structures that could be added to an environment to make it globally accessible to a robot.  In experiments using real and simulated environments, we demonstrated that the algorithm can synthesize optimal sets of structures with practical speed in realistic settings.  This opens up the possibility for a structure-building robot to enter a new environment and quickly determine what structures should be built to enable free movement, enabling tasks that would otherwise be very difficult or impossible for the robot.
\vspace{-0.12cm}
\bibliographystyle{IEEEtran}
\bibliography{references}

\end{document}

%% file: runtime_table.tex
\begin{center}
\begin{tabular}{lrrrrr}
\hline
\bf{Environment} & \bf{\exTableChair} & \bf{\exPushedIn} & \bf{\exCheckerboard} & \bf{\exOffice} & \bf{\exStairs}\\
\hline
\bf{Size (square, px)} & 151 & 151 & 72 & 240 & 115\\
\hline
\bf{Structures} & 634 & 148 & 474 & 3302 & 1275\\
\hline
\bf{Regions} & 4 & 4 & 9 & 5 & 6\\
\hline
\bf{Branches Explored} & 33 & 1 & 13 & 1 & 1\\
\hline
\bf{Potential Conflicts} & 16196 & 476 & 12974 & 33908 & 18206\\
\hline
\bf{CPF} & 14.0\% & 16.3\% & 7.6\% & 3.3\% & 4.7\%\\
\hline
\bf{Structure Time (s)} & 9.834 & 2.25 & 3.188 & 61.824 & 9.507\\
\hline
\bf{BB-MST Time (s)} & 6.642 & 0.024 & 9.588 & 0.552 & 0.14\\
\hline
\bf{Total Time (s)} & 20.604 & 6.329 & 13.828 & 86.848 & 15.708\\
\hline
\bf{Blocks used} & 17 & 46 & 40 & 13 & 15\\
\hline
\end{tabular}
\end{center}